\documentclass[10pt, conference, letterpaper]{IEEEtran}
\IEEEoverridecommandlockouts

\usepackage[english]{babel}

\usepackage{cite}
\usepackage{amsmath,amsfonts}
\usepackage{algorithmic,algorithm}
\usepackage{graphicx}
\usepackage{textcomp}
\usepackage{xcolor}
\usepackage{tikz}
\usetikzlibrary{patterns}
\usepackage{caption}
\usepackage{subcaption}
\usepackage{cleveref}

\newcommand\algorithmicprocedure{\textbf{procedure}}
\newcommand{\algorithmicendprocedure}{\algorithmicend\ \algorithmicprocedure}
\makeatletter
\newcommand\PROCEDURE[3][default]{%
	\ALC@it
	\algorithmicprocedure\ \textsc{#2}(#3)%
	\ALC@com{#1}%
	\begin{ALC@prc}%
	}
	\newcommand\ENDPROCEDURE{%
	\end{ALC@prc}%
	\ifthenelse{\boolean{ALC@noend}}{}{%
		\ALC@it\algorithmicendprocedure
	}%
}
\newenvironment{ALC@prc}{\begin{ALC@g}}{\end{ALC@g}}
\makeatother

\newtheorem{remark}{Remark}

\newtheorem{lemma}{Lemma}
\newtheorem{theorem}{Theorem}

\newenvironment{proof}{\begin{IEEEproof}}{\end{IEEEproof}}

\def\BibTeX{{\rm B\kern-.05em{\sc i\kern-.025em b}\kern-.08em
		T\kern-.1667em\lower.7ex\hbox{E}\kern-.125emX}}

\begin{document}
\title{Monitoring State Transitions in Markovian Systems with Sampling Cost} 

\author{\IEEEauthorblockN{Kumar Saurav}
	\IEEEauthorblockA{
		\textit{Shiv Nadar Institution of Eminence}\\
		Delhi-NCR, India \\
		kumar.saurav@snu.edu.in}
	\and
	\IEEEauthorblockN{Ness B. Shroff}
	\IEEEauthorblockA{\textit{The Ohio State University} \\
		\textit{Columbus, Ohio, USA}\\
		shroff.11@osu.edu.in}
	\and
	\IEEEauthorblockN{Yingbin Liang}
	\IEEEauthorblockA{\textit{The Ohio State University} \\
		\textit{Columbus, Ohio, USA}\\
		liang.889@osu.edu.in}
}
\vspace{-2in}

\def\onehalf{\frac{1}{2}}
\def\etal{et.\/ al.\/}
\newcommand{\bydef}{\triangleq}
\newcommand{\tr}{{\it{tr}}}
\def\SNR{{\textsf{SNR}}}
\def\bydef{:=}
\def\bba{{\mathbb{a}}}
\def\bbb{{\mathbb{b}}}
\def\bbc{{\mathbb{c}}}
\def\bbd{{\mathbb{d}}}
\def\bbee{{\mathbb{e}}}
\def\bbff{{\mathbb{f}}}
\def\bbg{{\mathbb{g}}}
\def\bbh{{\mathbb{h}}}
\def\bbi{{\mathbb{i}}}
\def\bbj{{\mathbb{j}}}
\def\bbk{{\mathbb{k}}}
\def\bbl{{\mathbb{l}}}
\def\bbm{{\mathbb{m}}}
\def\bbn{{\mathbb{n}}}
\def\bbo{{\mathbb{o}}}
\def\bbp{{\mathbb{p}}}
\def\bbq{{\mathbb{q}}}
\def\bbr{{\mathbb{r}}}
\def\bbs{{\mathbb{s}}}
\def\bbt{{\mathbb{t}}}
\def\bbu{{\mathbb{u}}}
\def\bbv{{\mathbb{v}}}
\def\bbw{{\mathbb{w}}}
\def\bbx{{\mathbb{x}}}
\def\bby{{\mathbb{y}}}
\def\bbz{{\mathbb{z}}}
\def\bb0{{\mathbb{0}}}

\def\bydef{:=}
\def\ba{{\mathbf{a}}}
\def\bb{{\mathbf{b}}}
\def\bc{{\mathbf{c}}}
\def\bd{{\mathbf{d}}}
\def\bee{{\mathbf{e}}}
\def\bff{{\mathbf{f}}}
\def\bg{{\mathbf{g}}}
\def\bh{{\mathbf{h}}}
\def\bi{{\mathbf{i}}}
\def\bj{{\mathbf{j}}}
\def\bk{{\mathbf{k}}}
\def\bl{{\mathbf{l}}}
\def\bm{{\mathbf{m}}}
\def\bn{{\mathbf{n}}}
\def\bo{{\mathbf{o}}}
\def\bp{{\mathbf{p}}}
\def\bq{{\mathbf{q}}}
\def\br{{\mathbf{r}}}
\def\bs{{\mathbf{s}}}
\def\bt{{\mathbf{t}}}
\def\bu{{\mathbf{u}}}
\def\bv{{\mathbf{v}}}
\def\bw{{\mathbf{w}}}
\def\bx{{\mathbf{x}}}
\def\by{{\mathbf{y}}}
\def\bz{{\mathbf{z}}}
\def\b0{{\mathbf{0}}}
\def\opt{\mathsf{OPT}}
\def\on{\mathsf{ON}}
\def\off{\mathsf{OF}}
\def\bA{{\mathbf{A}}}
\def\bB{{\mathbf{B}}}
\def\bC{{\mathbf{C}}}
\def\bD{{\mathbf{D}}}
\def\bE{{\mathbf{E}}}
\def\bF{{\mathbf{F}}}
\def\bG{{\mathbf{G}}}
\def\bH{{\mathbf{H}}}
\def\bI{{\mathbf{I}}}
\def\bJ{{\mathbf{J}}}
\def\bK{{\mathbf{K}}}
\def\bL{{\mathbf{L}}}
\def\bM{{\mathbf{M}}}
\def\bN{{\mathbf{N}}}
\def\bO{{\mathbf{O}}}
\def\bP{{\mathbf{P}}}
\def\bQ{{\mathbf{Q}}}
\def\bR{{\mathbf{R}}}
\def\bS{{\mathbf{S}}}
\def\bT{{\mathbf{T}}}
\def\bU{{\mathbf{U}}}
\def\bV{{\mathbf{V}}}
\def\bW{{\mathbf{W}}}
\def\bX{{\mathbf{X}}}
\def\bY{{\mathbf{Y}}}
\def\bZ{{\mathbf{Z}}}
\def\b1{{\mathbf{1}}}

\def\bbA{{\mathbb{A}}}
\def\bbB{{\mathbb{B}}}
\def\bbC{{\mathbb{C}}}
\def\bbD{{\mathbb{D}}}
\def\bbE{{\mathbb{E}}}
\def\bbF{{\mathbb{F}}}
\def\bbG{{\mathbb{G}}}
\def\bbH{{\mathbb{H}}}
\def\bbI{{\mathbb{I}}}
\def\bbJ{{\mathbb{J}}}
\def\bbK{{\mathbb{K}}}
\def\bbL{{\mathbb{L}}}
\def\bbM{{\mathbb{M}}}
\def\bbN{{\mathbb{N}}}
\def\bbO{{\mathbb{O}}}
\def\bbP{{\mathbb{P}}}
\def\bbQ{{\mathbb{Q}}}
\def\bbR{{\mathbb{R}}}
\def\bbS{{\mathbb{S}}}
\def\bbT{{\mathbb{T}}}
\def\bbU{{\mathbb{U}}}
\def\bbV{{\mathbb{V}}}
\def\bbW{{\mathbb{W}}}
\def\bbX{{\mathbb{X}}}
\def\bbY{{\mathbb{Y}}}
\def\bbZ{{\mathbb{Z}}}

\def\cA{\mathcal{A}}
\def\cB{\mathcal{B}}
\def\cC{\mathcal{C}}
\def\cD{\mathcal{D}}
\def\cE{\mathcal{E}}
\def\cF{\mathcal{F}}
\def\cG{\mathcal{G}}
\def\cH{\mathcal{H}}
\def\cI{\mathcal{I}}
\def\cJ{\mathcal{J}}
\def\cK{\mathcal{K}}
\def\cL{\mathcal{L}}
\def\cM{\mathcal{M}}
\def\cN{\mathcal{N}}
\def\cO{\mathcal{O}}
\def\cP{\mathcal{P}}
\def\cQ{\mathcal{Q}}
\def\cR{\mathcal{R}}
\def\cS{\mathcal{S}}
\def\cT{\mathcal{T}}
\def\cU{\mathcal{U}}
\def\cV{\mathcal{V}}
\def\cW{\mathcal{W}}
\def\cX{\mathcal{X}}
\def\cY{\mathcal{Y}}
\def\cZ{\mathcal{Z}}

\def\sfA{\mathsf{A}}
\def\sfB{\mathsf{B}}
\def\sfC{\mathsf{C}}
\def\sfD{\mathsf{D}}
\def\sfE{\mathsf{E}}
\def\sfF{\mathsf{F}}
\def\sfG{\mathsf{G}}
\def\sfH{\mathsf{H}}
\def\sfI{\mathsf{I}}
\def\sfJ{\mathsf{J}}
\def\sfK{\mathsf{K}}
\def\sfL{\mathsf{L}}
\def\sfM{\mathsf{M}}
\def\sfN{\mathsf{N}}
\def\sfO{\mathsf{O}}
\def\sfP{\mathsf{P}}
\def\sfQ{\mathsf{Q}}
\def\sfR{\mathsf{R}}
\def\sfS{\mathsf{S}}
\def\sfT{\mathsf{T}}
\def\sfU{\mathsf{U}}
\def\sfV{\mathsf{V}}
\def\sfW{\mathsf{W}}
\def\sfX{\mathsf{X}}
\def\sfY{\mathsf{Y}}
\def\sfZ{\mathsf{Z}}

\def\bydef{:=}
\def\sfa{{\mathsf{a}}}
\def\sfb{{\mathsf{b}}}
\def\sfc{{\mathsf{c}}}
\def\sfd{{\mathsf{d}}}
\def\sfee{{\mathsf{e}}}
\def\sfff{{\mathsf{f}}}
\def\sfg{{\mathsf{g}}}
\def\sfh{{\mathsf{h}}}
\def\sfi{{\mathsf{i}}}
\def\sfj{{\mathsf{j}}}
\def\sfk{{\mathsf{k}}}
\def\sfl{{\mathsf{l}}}
\def\sfm{{\mathsf{m}}}
\def\sfn{{\mathsf{n}}}
\def\sfo{{\mathsf{o}}}
\def\sfp{{\mathsf{p}}}
\def\sfq{{\mathsf{q}}}
\def\sfr{{\mathsf{r}}}
\def\sfs{{\mathsf{s}}}
\def\sft{{\mathsf{t}}}
\def\sfu{{\mathsf{u}}}
\def\sfv{{\mathsf{v}}}
\def\sfw{{\mathsf{w}}}
\def\sfx{{\mathsf{x}}}
\def\sfy{{\mathsf{y}}}
\def\sfz{{\mathsf{z}}}
\def\sf0{{\mathsf{0}}}

\def\Nt{{N_t}}
\def\Nr{{N_r}}
\def\Ne{{N_e}}
\def\Ns{{N_s}}
\def\Es{{E_s}}
\def\No{{N_o}}
\def\sinc{\mathrm{sinc}}
\def\dmin{d^2_{\mathrm{min}}}
\def\vec{\mathrm{vec}~}
\def\kron{\otimes}
\def\Pe{{P_e}}
\newcommand{\expeq}{\stackrel{.}{=}}
\newcommand{\expg}{\stackrel{.}{\ge}}
\newcommand{\expl}{\stackrel{.}{\le}}
\def\SIR{{\mathsf{SIR}}}

\def\nn{\nonumber}

\maketitle
\begin{abstract}
	We consider a node-monitor pair, where the node's state varies with time. The monitor needs to track the node's state at all times, however, there is a fixed cost for each state query.
	So the monitor may instead predict the state using time-series forecasting methods, including time-series foundation models (TSFMs), and query only when prediction uncertainty is high. Since query decisions influence prediction accuracy, determining when to query is nontrivial. A natural approach is a greedy policy that predicts when the expected prediction loss is below the query cost and queries otherwise. We analyze this policy in a Markovian setting, where the optimal (OPT) strategy is a state-dependent threshold policy minimizing the time-averaged sum of query cost and prediction losses. We show that, in general, the greedy policy is suboptimal and can have an unbounded competitive ratio, but under common conditions—such as identically distributed transition probabilities—it performs close to OPT. For the case of unknown transition probabilities, we further propose a projected stochastic gradient descent (PSGD)-based learning variant of the greedy policy, which achieves a favorable predict–query tradeoff with improved computational efficiency compared to OPT. 
\end{abstract}


\maketitle

\section{Introduction} \label{sec:introduction}

Consider a status update system, where a remote monitor needs to track the time-varying state of a node. Traditionally, in such systems, the monitor repeatedly queries the node, and maintains an accurate estimate of its current state at all times. However, with advent of large scale energy-constrainted IoT systems, it is critical to optimize the energy and channel usage. Therefore, in the modern paradigm, 
the monitor must use the available information to predict node's current state, and query only when the uncertainty in node's state is large. In this paper, we refer to it as \emph{interleaved query-predict approach (IQP)}.

IQP raises two coupled questions: \((i)\) how to predict, and \((ii)\) when to query.  A natural greedy rule predicts the state with minimum expected loss and queries only when that loss exceeds the query cost.  However, this ignores the informational value of queries: a query can reduce future prediction errors, so querying now may lower long-run cost even if its immediate cost exceeds the current expected loss.  While the greedy policy is simple and attractive under model uncertainty, an optimal policy may be computationally costly.  Hence we seek policies that balance near-optimal query costs and prediction loss with practical implementability.
Therefore, the decision problem needs to be analyzed thoroughly to find a balanced solution that is near-optimal in query cost and prediction loss, and easily implementable for wider adoption. 

With this motivation, in this paper, we consider the greedy policy introduced above, and compare its performance in relation to an optimal policy (OPT) that incurs minimum average cost (weighted sum of average query cost and average prediction loss). We seek to identify scenarios where the greedy policy (which is simple to implement) is a promising alternative to OPT in average cost. To this end, we consider a setting where the node's state varies following a finite-state Markov chain, and analyze the policies for scenarios where the state transition probabilities are i) known, and ii) unknown. 

The reason we consider a Markovian system is two-folds. First, it provides a structured setting where we can derive OPT and analyze its properties in relation to the greedy policy. Second, in this setting, OPT can make long term optimizations, and therefore has most advantage over the greedy policy. Thus, intuitively, it provides a bound on the optimality gap for the greedy policy in generic settings. 

Further, we adopt a linear combination of prediction loss and query cost as the objective, as it compactly captures the query–predict tradeoff and extends naturally to multi-node settings.
(i) With suitable weights \cite{saurav2024minimizing}, this formulation balances inversely related costs while satisfying all constraints.
(ii) It further allows derivation of Whittle’s index (WI) \cite{nino2023markovian}, enabling Whittle’s index policy (WIP) for efficient multi-node query scheduling to minimize average prediction error.
{\bf Related works.} Prior works on remote tracking have mostly focused on minimizing state error using metrics like age-of-information (AoI) \cite{kaul2012real,saurav2023online}, age-of-incorrect-information (AoII) \cite{joshi2021minimization}, and others \cite{salimnejad2025age}. These approaches typically optimize query timing to reduce delay between state changes and subsequent queries, subject to query cost constraints. However, treating all state errors equally can be misleading. For instance, in a safety application, failing to detect a fire is far more costly than a false alarm. Ignoring such inter-state criticality may lead to severe consequences. 

Another variant of remote tracking problem considered in literature is on estimation of remote process, where the goal is to minimize the mean squared error between the actual and the estimated process \cite{sun2019sampling,ornee2021sampling}. These works assume that the process follows simple dynamics such as Wiener or Ornstein-Uhlenbeck process. Therefore, the estimation error is correlated to the time since the last query, and minimizing the mean squared error reduces to minimizing AoI. Note that such an approach is not suitable for discrete state spaces where state transitions may be sudden, and cost due to state errors may vary significantly across different states.

Most relevant work on remote tracking with inter-state criticality is \cite{ornee2025remote}, which considers a multi-node setting where the state of the node varies following a finite-state Markov chain, and in each time slot, the monitor must choose the subset of nodes to query such that, under a constraint on the maximum number of simultaneous queries, the average state error across all nodes is minimized. With some relaxations, this problem reduces to the single-node setting considered in this paper, and assuming the state transition probabilities are known, \cite{ornee2025remote} models it as a Markov decision process (MDP), and solves it using value iteration. However, this needs to be repeated for each time slot, making the solution computationally expensive. Hence, we need a simpler alternative such as a greedy policy. 

An important aspect of the considered problem is the knowledge of the state transition probability matrix $\cP$ which determines the uncertainty in future states. Though prior works like \cite{sun2019sampling,ornee2021sampling,ornee2025remote} assume $\cP$ to be known, in practice, $\cP$ may have to be inferred/learnt from observations. A popular approach is to use the classical stopping-time technique \cite{hao2018learning} where a policy queries in first $\alpha$ time slots (for some $\alpha>0$) and estimate one-step transition probability $\hat{p}_{ij}$ (from state $s_i$ to $s_j$) by computing the ratio of number of transitions from state $s_i$ to $s_j$ and the number of visits to state $s_i$. However, as we discuss in Section \ref{sec:unknownP}, this approach is inefficient for estimating $n-$step transition probability ($n\gg 1$) as is required for the considered problem. This adds another dimension to the considered problem calling for a more efficient approach to learn 
$\cP$ with only few intermittent queries. 


{\bf Contributions.} In this paper, we derive OPT 
for the considered problem with known $\cP$, and compare the greedy policy $\pi^g$ (and some other heuristic policies) with OPT. We show that in certain scenarios $\pi^g$'s performance can be arbitrarily bad compared to OPT, and propose a fix for such an issue. We also identify settings where $\pi^g$ has near-optimal performance, and is preferable over OPT due to ease of implementation (e.g. when transition probabilities are unknown). 

\section{System Model} \label{sec:sysmodel}

Consider a discrete-time setting where the state of a node varies with time following some finite (discrete)-state Markov chain $\cM$, and a monitor needs to track the node's current state at all times. To track the state, the monitor can either query (sample) the current state of the node, or predict it based on the previously queried state.

Let $\cS=\{s_1,s_2,...,s_K\}$ (for finite integer $K>0$) denote the set of all possible states, and $\cP$ be the state transition probability matrix for $\cM$. 
At time $t$, if the current state of the node is $X_t\in \cS$, and the monitor predicts the state to be $\hat{X}_t\in\cS$, then there is a state error (prediction loss) equal to $L(X_t,\hat{X}_t)\ge 0$, where $L(\cdot,\cdot)$ is the loss function. If the monitor queries the state at time $t$, then the state error is $0$, and the monitor receives the actual state $X_t$ (i.e. $\hat{X}_t=X_t$). However, it incurs a fixed query (sampling) cost $c\ge 0$. 

\begin{remark} \label{remark:notation-loss-matrix}
    We assume that the loss $L(X_t,\hat{X}_t)$ is deterministic for all possible combinations of $X_t,\hat{X}_t\in\cS$, and recorded as a loss matrix $\cL\in\bbR^{K\times K}_{\ge 0}$, with $(i,j)-$th element $\ell_{ij}=L(s_i,s_j)$ $\forall i,j\in\{1,2,...,K\}$. 
\end{remark}
 
The goal is to find an online (causal) policy $\pi$ that at each time $t$, decides whether to query or predict, and if predict, then how to choose $\hat{X}_t$ such that the time-averaged sum of the expected prediction loss and the query cost (in short, \emph{sum average cost} $\Gamma(\pi)$) is minimized.
 Succinctly, this is given as
\begin{align} \label{eq:objective}
    \underset{\pi\in\Pi}{\min} \ \Gamma(\pi) \equiv \underset{\pi\in\Pi}{\min} \ \lim_{T\to\infty}\frac{1}{T}\sum_{t=0}^T\bbE[L(X_t,\hat{X}^\pi_t)+c\cdot Q^\pi_t], 
\end{align}
where $\Pi$ denotes the set of all online policies $\pi$, while $Q^\pi_t$ denotes the number of queries made until time $t$ under policy $\pi$. In this paper, we consider two variants of the above problem: (i) when $\cP$ is known at the monitor, and (ii) when $\cP$ is unknown at the monitor. First, we focus on the former case, in Section \ref{sec:knownP}, and analyze the performance of a greedy policy $\pi^g$ against an optimal policy $\pi^\star$ for \eqref{eq:objective}. Subsequently, in Section \ref{sec:unknownP} we extend $\pi^g$ to learn $\cP$, and discuss its practicality over $\pi^\star$ due to ease of implementation and low-complexity.

\begin{remark} \label{remark:notation-p-hat-tau}
    In this paper, let \( p_{ij}^{(n)} \) denote the \( n \)-step transition probability from state \( s_i \) to \( s_j \), where \( p_{ij}^{(1)} = p_{ij} \); equivalently, \( p_{ij}^{(n)} \) is the \((i,j)\)-th entry of \( \cP^n \), the \(n\)-th power of the transition matrix \( \cP \). Also, the hat symbol (e.g., \( \hat{X} \)) denotes an estimate or prediction of a variable, and \( \tau(t) \) represents the most recent query time before \( t \).
\end{remark}

\section{Case $1$: Transition Probabilities are Known} \label{sec:knownP}

At any time $t$, given $\cP$ and the previously queried state $X_{\tau(t)}$ (at time $\tau(t)$) a critical question is: how to predict the current state $X_t$? We answer this in the following Lemma \ref{lem:optpred}. 

\begin{lemma} \label{lem:optpred}
    Let $n_t=t-\tau(t)$, i.e. the time since last query until the current time $t$. Also, let the queried state at time $\tau(t)$ be $s_i$.
    At time t, the optimal prediction strategy is to pick state 
    \begin{align} \label{eq:optpred}
        s_{k^\star} = \arg \min_{s_k\in\cS} \sum_{s_j\in\cS} p_{ij}^{(n_t)}\cdot \ell_{jk},
    \end{align}
    where $\ell_{jk}$ is the loss when the state is $s_j$ while the prediction is $s_k$. Also, the optimal (expected) prediction loss is $\sum_{s_j\in\cS} p_{ij}^{(n_t)}\ell_{jk^\star}$.
\end{lemma}
\begin{proof}
    Let an optimal prediction policy $\pi^p$ predicts state $\hat{X}_t=s_k$ with probability $q_{ik}^{(n_t)}$ ($\forall k\in\cS$). Since the probability $P(X_t=s_j|X_{\tau(t)}=s_i)=p_{ij}^{(n_t)}$, the expected prediction loss for $\pi^p$ is $\sum_{s_k\in\cS} q_{ik}^{(n_t)}\cdot \sum_{s_j\in\cS} p_{ij}^{(n_t)}\cdot \ell_{jk}$.  
    Clearly, this expected prediction loss for $\pi^p$ is minimum if $q_{ik}^{(n_t)}=1$ for $s_k=s_{k^\star}$ \eqref{eq:optpred}, and $q_{ik}^{(n_t)}=0$ for $s_k\neq s_{k^\star}$. Also, the corresponding minimum prediction loss is equal to $\sum_{s_j\in\cS} p_{ij}^{(n_t)}\ell_{jk^\star}$. 
\end{proof}

The next step is to design a policy that decides whether to predict \eqref{eq:optpred} or query in each time slot such that the average sum cost \eqref{eq:objective} is minimized. To this end, a natural approach is to follow a greedy policy $\pi^g$ (Algorithm \ref{algo:greedy-policy}) that at any time, predicts if the expected prediction loss for \eqref{eq:optpred} is less than the query cost $c$, and queries otherwise. 

\begin{algorithm} [htb]
	\caption{Greedy policy $ \pi^g$.}
	\label{algo:greedy-policy}
\begin{algorithmic}[1]
    \STATE let $s_i$ denote the queried state at time $\tau(t)$, $\forall \ t\ge 1$;
    \FOR{time $t=1,2,3,\cdots$}
        \STATE compute $n_t=t-\tau(t)$; 
        \IF{$\sum_{s_j\in\cS} p_{ij}^{(n_t)}\ell_{jk^\star} < c$} 
        \STATE predict state $\hat{X}_t=s_k^\star$ \eqref{eq:optpred};
        \ELSE
        \STATE query the current state $X_t$;
        \ENDIF
    \ENDFOR
\end{algorithmic}
\end{algorithm}

Intuitively, it is always better to query if the expected prediction loss is higher than (or equal to) the query cost. However, under $\pi^g$, always predicting when the expected prediction loss is less than the query cost can be suboptimal. This is because by revealing current state, a query may help minimize future prediction losses as well, as shown next. 

\begin{theorem} \label{thm:greedy-bad}
    When Markov chain $\cM$ has absorbing states, or in general, states with very little uncertainty in future state trajectory (offering low prediction loss over a period of time), the ratio of the sum average cost \eqref{eq:objective} for $\pi^g$ to that of an optimal policy $\pi^\star$ can be arbitrarily large.
\end{theorem}

\begin{proof} 
    Consider a node with states $\cS=\{s_1,s_2,s_3\}$, and matrices $\cP=\begin{bmatrix} 0 & 0.5 & 0.5 \\ 0 & 1 & 0 \\ 0 & 0 & 1 \end{bmatrix}$ and $\cL=\begin{bmatrix} 0 & 1 & 1 \\ 1 & 0 & 1 \\ 1 & 1 & 0 \end{bmatrix}$. Also, let the initial state $X_0=1$, and the query cost $c=1$. 
    In this setting, note that $\cP$ is idempotent, i.e. $\cP^n=\cP$, and $\sum_{s_j\in\cS} p_{1j}^{(n)}\ell_{jk^\star}=0.5<1=c$, $\forall \ n\ge 1$. Therefore, the greedy policy $\pi^g$ never queries, and incurs an expected prediction loss of $0.5$ in each time slot, leading to the sum average cost \eqref{eq:objective} equal to $0.5$. 
    However, note that $\cP$ is such that starting in state $X_0=1$, in slot $t=1$, the node  transitions either to state $2$ or state $3$ with equal probability, and then remains in that state forever thereafter. 
    Therefore, an optimal policy $\pi^\star$ queries in slot $1$ (incurring a query cost $c=1$), and then predicts in all future time slots as the prediction loss will be $0$. Hence, as $t\to\infty$, the expected sum average cost \eqref{eq:objective} for  $\pi^\star$ approaches $0$.
    Thus, computing the ratio of the expected sum average costs \eqref{eq:objective} for $\pi^g$ and $\pi^\star$, we get the result. 
\end{proof}

The question is, how $\pi^g$ can be made robust against scenarios in Theorem \ref{thm:greedy-bad}, and in what scenarios is $\pi^g$ close to optimal? We analyze this numerically in Section \ref{sec:numerical-results} by comparing the performance of $\pi^g$ with $\pi^\star$, derived next.

\subsection*{An optimal online policy $\pi^\star$ for Markovian systems}
\begin{lemma} \label{lem:optpolicy-struct}
    For the considered Markovian system, there exists an optimal online policy $\pi^\star$ such that i) $\pi^\star$ is stationary, and ii) if the latest query has been in slot $\tau$, revealing the current state $X_\tau=s_i$, $\pi^\star$ queries next in slot $\tau+\mu_i^\star$, where $\mu_i^\star$ is a deterministic positive integer that only depends on $s_i$. Also, in each slot  $t\in\{\tau+1, \tau+2, \cdots , \tau+(\mu_i^\star-1)\}$), $\pi^\star$ predicts following the optimal prediction strategy \eqref{eq:optpred}.
\end{lemma}

\begin{proof}
    In any slot $\tau$, if the current state $X_\tau=s_i$ is known, then because of the Markovian property of state transition process, the distribution over the future state trajectories depends only on $s_i$ (i.e. independent of the past states or the time index $\tau$). Moreover, in subsequent time slots, until the state is queried again, there is no external input or additional information that can influence this distribution. Therefore, there must exist an optimal online policy $\pi^\star$ that chooses the next query time deterministically, depending only on $s_i$. This also implies that $\pi^\star$'s decision is independent of $\tau$, i.e. $\pi^\star$ is stationary. 
    
    Further, since $\pi^\star$ does not query at time $t\in\{\tau+1, \tau+2, \cdots, \tau+(\mu_i^\star-1)\}$ (by definition of $\mu_i^\star$), it must predict following the optimal prediction strategy \eqref{eq:optpred} to minimize the expected prediction loss. 
\end{proof}

Lemma \ref{lem:optpolicy-struct} implies that a state-dependent threshold policy $\pi\in\Pi$ (Algorithm \ref{algo:optimal-policy}) that at any time $t$, if the queried state is $s_i\in\cS$, then predicts $k^\star$ \eqref{eq:optpred} at times $t+1,t+2,\cdots,t+\mu_i^\pi-1$, and queries next at time $t+\mu_i^\pi$, is optimal for some threshold $\mu_i^\pi=\mu_i^\star$, $\forall i\in\cS$. 
Therefore, we next derive $\mu_i^\star$'s.

\begin{algorithm} [htb]
	\caption{State-dependent threshold policy $\pi\in\Pi$.}
	\label{algo:optimal-policy}
\begin{algorithmic}[1]
    \STATE compute $\mu^\pi=\{\mu_i^\pi>0: \forall \ i\in\cS\}$;   \ \  // see Lemma \ref{lemma:opt-condtion} 
    \STATE query initial state $X_0$;
    \STATE $\hat{t}=\mu^\pi_{X_0}$; \ \  // compute next query time
    \FOR{time $t=1,2,3,\cdots$}
        \IF{$t\ne\hat{t}$} 
        \STATE predict state $\hat{X}_t=s_{k^\star}$ \eqref{eq:optpred};
        \ELSE
        \STATE query current state $X_t$;
        \STATE $\hat{t}=t+\mu^\pi_{X_t}$; \ \ // update next query time
        \ENDIF
    \ENDFOR
\end{algorithmic}
\end{algorithm}

For policy $\pi\in\Pi$, let $\tau_n^\pi$ denote the $n^{th}$ query time, and $X_{\tau_n}^\pi$ be the corresponding queried state. Note that the time-axis can be partitioned into intervals between successive query times under $\pi$. We refer to each of these intervals as a \emph{stage} for $\pi$, with the $n^{th}$ stage $\chi_n^\pi=\{\tau_n^\pi,\cdots,\tau_{n+1}^\pi-1\}$, where $\tau_{n+1}^\pi=\tau_n^\pi+\mu_{X_{\tau_n}^\pi}^\pi$. 
In each stage, there is exactly one query (at the start of the stage), followed by prediction in the remaining time slots of the stage. Also, due to the Markovian property, under policy $\pi$, the prediction losses in the stage only depends on the state $s_i$ revealed after the query (\emph{queried state}). Therefore, we denote the expected total cost incurred in a stage with queried state $s_i$ as $\bbE[\Lambda(s_i, \mu_i^\pi)]$, where $\mu_i^\pi$ is the threshold for state $s_i$ under state-dependent threshold policy $\pi$. Furthermore, the relative cost in the stage for policy $\pi$ (relative to optimal policy $\pi^\star$) is $\bbE[\overline{\Lambda}^\pi(s_i)]=\bbE[\Lambda(s_i, \mu_i^\pi)]-\Gamma(\pi^\star)\mu_i^\pi$, where $\Gamma(\pi^\star)$ is the sum average cost \eqref{eq:objective} for an optimal policy $\pi^\star$. 

Note that under policy $\pi$, the queried state $X_{\tau_{n}^\pi}$ in $i^{th}$ stage only depends on the queried state $s_i$ in the previous stage. In particular, $X_{\tau_{n}^\pi}=s_j$ with probability $p_{ij}^{(\mu_i^\pi)}$, where the superscript $(\mu_i^\pi)$ denotes $\mu_i^\pi$ step transition probability.
Therefore, the relative cost $h^\pi(s_i)$ incurred under policy $\pi$ starting from the stage with queried state $s_i$ is 

\begin{align}
    \label{eq:relative-cost-function-full}
    h^\pi(s_i)&= 
    \bbE[\Lambda(s_i, \mu_i^\pi)]-\Gamma(\pi^\star)\mu_i^\pi + \sum_{j\in\cS}p_{ij}^{(\mu_i^\pi)} h^\pi(s_j).
\end{align} 

Without loss of generality, let the queried state in first/initial stage be $s_1$. Then under an optimal policy $\pi^\star$, $h^{\pi^\star}(s_1)=0$. Using this, we get the following result.
\begin{lemma} \label{lemma:opt-condtion}
    A policy $\pi\in\Pi$ is an optimal solution of problem \eqref{eq:objective} if and only if $h^\pi(s_1)=0$. 
\end{lemma} 

Note that in $h^\pi(s_1)$ there are $K+1$ unknowns: $\mu_i^\pi$'s (for $i=1,2,\cdots,K$) and $\Gamma(\pi^\star)$ (other terms including $h^\pi(s_i)$'s are functions of these unknowns). Thus, combining the optimality condition $h^\pi(s_1)=0$ along with the $K$ equations in \eqref{eq:relative-cost-function-full} (for $i=1,2,\cdots,K$), and using iterative methods, such as policy iteration \cite{bertsekas2012dynamic}, we get $\mu_i^\star$'s.
\begin{theorem} \label{thm:opt-using-PI}
    On solving 
     $h^\pi(s_1)=0$ along with the $K$ equations in \eqref{eq:relative-cost-function-full} (for $i=1,2,\cdots,K$) using policy iteration \cite{bertsekas2012dynamic}, we get the optimal thresholds
     $\mu_i^\star$'s for the state-dependent threshold policy (Algorithm \ref{algo:optimal-policy}), thus, revealing $\pi^\star$. 
\end{theorem}

\section{Case $2$: Transition Probabilities are Unknown} \label{sec:unknownP}

In previous section, we assumed the state transition probability matrix $\cP$ to be known, and used it to compute the optimal prediction loss in Lemma \ref{lem:optpred}, and formulate the relative cost function \eqref{eq:relative-cost-function-full}. However, in practice, $\cP$ may be unknown. In such cases, one possible option is model-free learning where a policy chooses actions to minimize the sum average cost \eqref{eq:objective} directly by trial-and-error without explicitly learning $\cP$. 
However, this approach suffers a critical challenge because the loss is not revealed in case of a prediction.  
Moreover, the cost function (prediction loss \eqref{eq:optpred}) at time $t$ depends on $t-\tau(t)$ (the time since last query). Therefore, under the model-free learning approach, a policy needs to learn to optimize the cost function \eqref{eq:optpred}
for each time slot $\tau+1, \tau+2, \cdots$, separately. 

Alternatively, a policy could learn the probability matrix $\cP$ (estimate $\hat{\cP}$) and use it to compute the cost function \eqref{eq:optpred} (using $(\hat{\cP})^n$) for all time slots $\tau+n$ ($n\ge 1$) as in the previous section. However, to compute $\hat{\cP}$, the classical stopping-time technique \cite{hao2018learning} where a policy queries in first $\alpha$ time slots (for some $\alpha>0$) and computes probability $\hat{p}_{ij}$ by computing the ratio of number of transitions from state $s_i$ to $s_j$ and the number of visits to state $s_i$ is inefficient. This is because the error in $\hat{\cP}$ amplifies while computing $\hat{\cP}^2, \hat{\cP}^3, \cdots$ (at time $\tau+2, \tau+3, \cdots$ respectively), thus requiring $\alpha$ to be very large.
Therefore, in this section, we propose an alternate approach for updating $\hat{\cP}$ using \emph{projected stochastic gradient descent (PSGD)} \cite{lacoste2012simpler}. 

Broadly, the PSGD-based approach (Algorithm \ref{algo:PSGD}) is as follows: Whenever there is a query at time $\tau+n$ ($\forall \ n\ge 1$), we compute a loss as a function of the state predicted using current estimate $(\hat{\cP})^n$ and the queried state $X_{\tau+n}$. Subsequently, we compute the gradient of the loss with respect to (w.r.t.) $\hat{\cP}$, and update the current estimate of $\hat{\cP}$ using the gradient. Since the row-sum of matrix $\hat{\cP}$ must be $1$, we finally project $\hat{\cP}$  (reset its negative elements to zero, and then normalize each element by the row-sum) to ensure that the row-sum constraint is satisfied. 

\begin{algorithm} [htb]
	\caption{PSGD to update $\cP$ estimate after each new query.}
	\label{algo:PSGD}
\begin{algorithmic}[1]
    \STATE latest query time $\tau$ and queried state $X_\tau=s_i$; 
    \FOR{time $t=\tau+1,\tau+2,\tau+3,\cdots$}
        \IF{query at time $t$ with state $X_t=s_j$} 
        \STATE define $y_k=1$ if $k=j$ and $y_k=0$ otherwise;
        \STATE compute $\hat{p}_{ik}^{(n)}$, $\forall k$, using current estimate $\hat{\cP}$;
        \STATE compute loss $\cF_{i,n}=\sum_{k=1}^K\left(y_k-\hat{p}_{ik}^{(n)}\right)^2$;
        \STATE compute gradient $\nabla \cF_{i,n}$ w.r.t. $\hat{\cP}$;
        \STATE update $\hat{\cP}=\hat{\cP}-\eta_t\nabla \cF_{i,n}$, where $\eta_t$ is the learning rate;
        \STATE project  
        $\hat{\cP}$ so that the row-sum of $\hat{\cP}$ is $1$;
        \STATE update $\tau=t$ and $X_\tau=s_j$; 
        \ELSE
        \STATE continue;
        \ENDIF
    \ENDFOR
\end{algorithmic}
\end{algorithm}

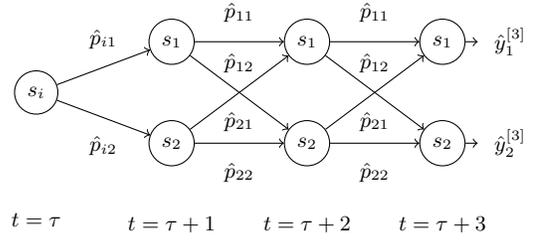
\begin{figure}
\centering
    \begin{tikzpicture}[scale=0.9, transform shape,
        every node/.style={circle, draw, minimum size=5mm},
        every path/.style={draw, ->},
        font=\small
    ]
    \tikzset{My Style/.style={draw=none, fill=none}} 
    
    \def\vs{1.5}

    \foreach \layer in {1,2,3} {
        \foreach \i in {1,2} {
            \node (L\layer\i) at (\layer*2, -\i*\vs) {$s_{\i}$};
        }
    }

    \node (I1) at (0, -1.5*\vs) [circle, draw, minimum size=5mm] {$s_{i}$}; 
    \draw (I1) -- node[My Style, midway, above] {$\hat{p}_{i1}$} (L11); 
    \draw (I1) -- node[My Style, midway, below] {$\hat{p}_{i2}$} (L12);
    \draw (I1) node[My Style, below=1.35cm]{$t=\tau$};

    \foreach \i in {1,2} {
        \node (I\i) at (7, -\i*\vs) [My Style] {$\hat{y}_{\i}^{[3]}$}; 
        \draw (L3\i) -- (I\i); 
    }

    \foreach \i in {1} {
        \foreach \j in {1,2} {
            \draw (L1\i) -- node[My Style, midway, above] {$\hat{p}_{\i\j}$} (L2\j);
        }
    }

     \foreach \i in {1} {
        \foreach \j in {1,2} {
            \draw (L2\i) --  node[My Style, midway, above] {$\hat{p}_{\i\j}$} (L3\j);
        }
    }
    \foreach \i in {2} {
        \foreach \j in {1,2} {
            \draw (L1\i) -- node[My Style, midway, below] {$\hat{p}_{\i\j}$} (L2\j);
        }
    }
    \foreach \i in {2} {
        \foreach \j in {1,2} {
            \draw (L2\i) -- node[My Style, midway, below] {$\hat{p}_{\i\j}$} (L3\j);
        }
    }
    \draw (L12) node[My Style, below=0.4cm]{$t=\tau+1$};
    \draw (L22) node[My Style, below=0.4cm]{$t=\tau+2$};
    \draw (L32) node[My Style, below=0.4cm]{$t=\tau+3$};
    \end{tikzpicture}
    \vspace{-1.7em}
    \caption{A neural network representation of the state transition process. This graph can be used to compute the error in state predictions, as well as the gradient of the loss w.r.t. probability estimates $\hat{p}_{ij}$. \vspace{-3ex}} 
    \label{fig:neural-network}
\end{figure}

\begin{remark}
    The gradients of the loss $\cF_{i,n}$ w.r.t. $\hat{p}_{ij}$ can be computed efficiently using the chain rule. In interpreting the state transitions over time as a feedforward neural network (Figure \ref{fig:neural-network}), where each node in the $m^{th}$ layer ($m=0,1,\cdots,n$) represents the state of the Markov chain at time $t=\tau+m$, and the edges represent the transition probabilities, we can compute the loss $\cF_{i,n}$ as a function of the estimated transition probabilities $\hat{p}_{ij}$. The gradient of the loss w.r.t. $\hat{p}_{ij}$ ($\forall \ m=1,2,\cdots,n$) can be computed using backpropagation. Note that the gradient w.r.t. $\hat{p}_{ij}$ in each layer can be different. The actual gradient of the loss $\cF_{i,n}$ w.r.t. $\hat{p}_{ij}$ is equal to the sum of the gradients of the loss w.r.t. $\hat{p}_{ij}$ in each layer.
\end{remark}

A major advantage of the PSGD approach is that it allows querying in non-consecutive time slots. Note that if the state at time $\tau$ is $s_i$ and we query at time $\tau+1$, the new state only reveals information about the transition probabilities $p_{ij}$ from state $i$. Therefore, despite being extremely useful in computing initial estimates of $\cP$, they turn out expensive (require too many queries) while refining  transition probability estimates, especially from states that are visited less frequently.
In contrast, note that $p_{1j}^{(n)}$ for $n\ge 3$ depends on all elements of $\cP$. Hence, using PSGD we can query at any time $\tau+n$ for $n\ge 3$, and refine $p_{ij}$ estimates $\forall \ i,j=\{1,2,\cdots,K\}$, easily.

\begin{remark}
    We use PSGD to minimize loss in transition probability estimates $\hat{\cP}$ instead of prediction loss \eqref{eq:optpred}  
    because the prediction loss is a function of $t-\tau$ (time since last query), and the minimizer of prediction loss can vary with time. In contrast, the minimizer of loss $\cF_{i,n}$ is $\cP$ for all $n\ge 1$.
\end{remark}

Though PSGD approach is useful in estimating the transition probability matrix $\cP$, the overall performance in minimizing the sum average cost \eqref{eq:objective} still depends on when the queries are made. Therefore, we must combine PSGD with some policy that at each time, uses the $\hat{\cP}$ estimates computed using PSGD and make query/predict decision optimizing \eqref{eq:objective}.  
However, this could be a risky approach because if a crude probability estimate $\hat{\cP}$ suggests to never query, then $\hat{\cP}$ will never be updated, and the policy will be stuck with a suboptimal decision. Therefore, we need to ensure that the policy queries at least once in every $N$ time slots (for some finite $N\ge 1$), and then use the PSGD approach to update $\hat{\cP}$. Finding such an optimal online policy tailored with PSGD updates is a challenging problem, and is a part of our ongoing work. Meanwhile, we may use PSGD with the greedy policy $\pi^g$ (Algorithm \ref{algo:greedy-policy}), modified to query whenever the last query is $N$ time slots old. We analyze the resulting PSGD+Greedy policy $\pi^{pg}$ (Algorithm \ref{algo:PSGD-greedy}) in Section \ref{sec:numerical-results} using numerical simulations. 

\begin{algorithm} [htb]
	\caption{PSGD+Greedy policy $ \pi^{pg}$.}
	\label{algo:PSGD-greedy}
\begin{algorithmic}[1]
    \STATE initialize maximum inter-query time $N\ge 1$ (finite); 
    \FOR{time $t=1,2,3,\cdots$}
    \STATE use $\pi^g$ (Algorithm \ref{algo:greedy-policy}) with current estimate $\hat{\cP}$ to decide whether to predict or query; 
    \STATE if query, update $\hat{\cP}$ using PSGD (Algorithm \ref{algo:PSGD});
    \ENDFOR
\end{algorithmic}
\end{algorithm}
\begin{remark}
    We use PSGD with \( \pi^g \) instead of \( \pi^\star \) (Algorithm~\ref{algo:optimal-policy}) since the predict/query decisions in \( \pi^\star \) depend intricately on all entries of \( \cP \), and even small estimation errors can degrade its performance. Moreover, \( \pi^\star \) requires running policy iteration (Theorem~\ref{thm:opt-using-PI}) after each query (i.e., after every \( \hat{\cP} \) update), making it computationally expensive. In contrast, the PSGD+Greedy policy \( \pi^{pg} \) (Algorithm~\ref{algo:PSGD-greedy}) is significantly simpler to implement.
\end{remark}

The following result shows that as time $t\to\infty$,  under $\pi^{pg}$, $\hat{\cP}$ converges to $\cP$ ($\pi^{pg}$ converges to the greedy policy $\pi^g$). 
\begin{theorem} \label{thm:PSGD-convergence}
    \label{prop:PSGD-convergence}
    For PSGD, let the learning rate for $m^{th}$ update step be $\eta_m=1/(8NK+m)$, where $N$ denotes the maximum inter-query time under policy $\pi^{pg}$ (Algorithm \ref{algo:PSGD-greedy}) and $K$ is the number of states in $\cS$. Under $\pi^{pg}$, as $t\to\infty$, the estimate $\hat{\cP}$ converges to $\cP$, and $\pi^{pg}$ converges to the greedy policy $\pi^g$ (Algorithm \ref{algo:greedy-policy}) with probability $1$. 
\end{theorem}

\begin{proof}[Proof Sketch]
    Since the maximum inter-query time $N$ is finite, as $t\to\infty$, the number of queries (and $\hat{\cP}$ updates) made by policy $\pi^{pg}$ also approaches infinity.
    The key idea in the proof is to show that as the number of iteration $m\to\infty$, the distance between $\hat{\cP}$ and $\cP$ converges to zero. The proof involves showing: i) the loss function $\cF_{i,n}$ is convex w.r.t. $p_{ij}$, $\forall \ s_i,s_j\in\cS$, and $n\ge 1$, and ii) with convex  $\cF_{i,n}$ and $\eta_m=1/(8NK+m)$, each PSGD step ($\hat{\cP}$ update) decreases the expected (norm) distance between $\hat{\cP}$ and $\cP$. 
\end{proof}

\section{Numerical Results} \label{sec:numerical-results}
Next, we analyze the algorithms and results discussed so far, along with relevant heuristics using numerical simulations.

\begin{remark} \label{rem:max-inter-query-time}
	From Theorem \ref{thm:greedy-bad}, we know that for certain transition probabilities $\cP$, the performance of $\pi^g$ can be arbritrarily bad, essentially due to aggressive prediction. Therefore, in this section, we limit the maximum inter-query time for $\pi^g$ (and other heuristic policies) to $N=10$ time slots. 
\end{remark}

Consider a Markov chain $\cM$ with $K=5$ states, $\cP=\begin{bmatrix} 0.5 & 0.5 & 0 & 0 & 0 \\ 0.1 & 0.1 & 0.6 & 0.2 & 0 \\ 0.2 & 0.1 & 0 & 0.5 & 0.2 \\ 0.1 & 0.2 & 0.2 & 0 & 0.5 \\ 0.1 & 0.1 & 0.2 & 0.3 & 0.3 \end{bmatrix}$, and $\cL=\begin{bmatrix} 0 & 1 & 2 & 3 & 4 \\ 1 & 0 & 1 & 2 & 3 \\ 2 & 1 & 0 & 1 & 2 \\ 3 & 2 & 1 & 0 & 1 \\ 4 & 3 & 2 & 1 & 0 \end{bmatrix}$. For $\cM$, we simulate policies $\pi^\star$, $\pi^g$ and $\pi^{pg}$ (Algorithm \ref{algo:PSGD-greedy}), along with two heuristic policies: i) uniform sampling policy $\pi^{us}$, which queries at a fixed time interval $\Delta^{us}=2$ (predicts $k^\star$ \eqref{eq:optpred} otherwise), and ii) a policy $\pi^{sd}$, which is similar to $\pi^g$, except that for prediction it uses stationary distribution of the Markov chain $\cM$ instead of the transition probabilities from the last queried state. Note that for considered $\cP$, the stationary distribution $p_{ij}^{(\infty)}=0.2$, $\forall $$s_i,s_j\in\cS$. 

We simulate the policies for different values of query cost $c$, and compute their sum average cost $\Gamma(\pi)$ by averaging over $100000$ time slots. The results are shown in Figure \ref{fig:recurrent_vsQcost}. As expected, when the query cost is small, except for $\pi^{us}$ that queries at fixed interval $\Delta^{us}$, all other policies query more aggressively, and perform identically. Likewise, when the query cost is too large, policies rely mostly on predictions, again leading to similar performance. However, with moderate query cost (i.e., $c\in[1.25, 1.5]$), $\pi^\star$ outperforms all others.

Recall that $\pi^{pg}$ is expected to converge to $\pi^g$. However, in Figure \ref{fig:recurrent_vsQcost}, for large $c$, there is a slight gap in their performance. This is because when $c$ is large, there are fewer queries, and the estimate $\hat{\cP}$ is updated less frequently. Therefore, when $c$ is large, convergence requires longer simulations in time.

\begin{figure} 
	\centerline{\includegraphics[width=0.8\linewidth]{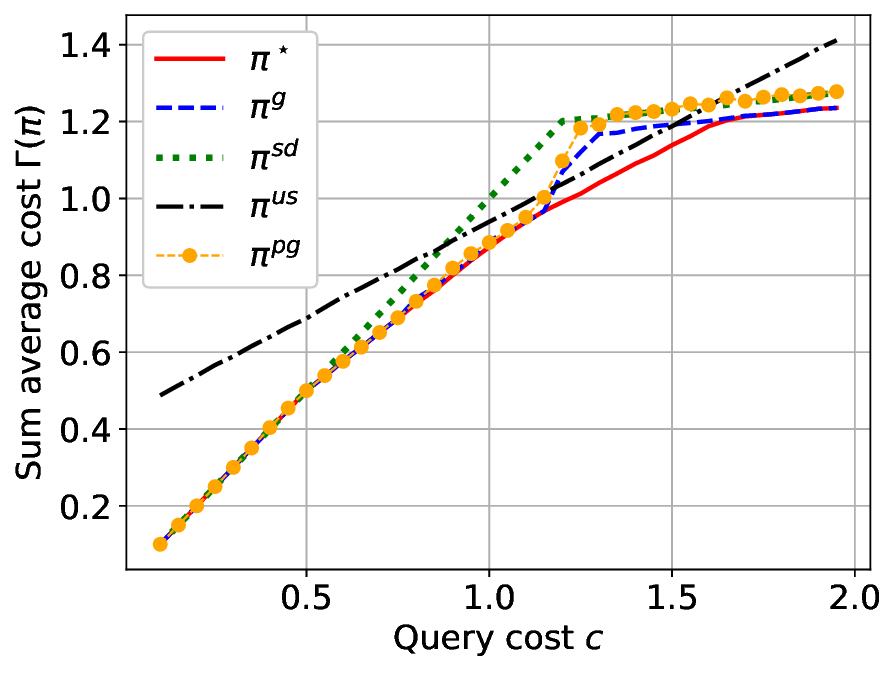}}
    \vspace{-1em}
	\caption{Performance of policies w.r.t. query cost $c$ in a recurrent Markov chain.\vspace{-2ex}}
	\label{fig:recurrent_vsQcost}
\end{figure}

From Figure \ref{fig:recurrent_vsQcost}, we observe that the performance of the uniform sampling policy $\pi^{us}$ strictly depends on the query cost $c$, and its performance can be poor if the sampling interval $\Delta^{us}$ is not optimized for $c$. For the above simulation, we optimized $\Delta^{us}$ for $c=1.4\in[1.25,1.5]$ where the performance of the policies differ the most. Figure \ref{fig:recurrent_vsUniformThresh} shows the performance of $\pi^{us}$ under the same setting (as previous simulation) with $c=1.4$ for different values of $\Delta^{us}$. We find that for certain values of $\Delta^{us}$, $\pi^{us}$ outperforms the greedy policy $\pi^g$, and $\pi^{sd}$. 
However, $\pi^{us}$ is never able to match the performance of $\pi^\star$ (Algorithm \ref{algo:optimal-policy}) which has state-dependent sampling intervals. 

Note that Figure \ref{fig:recurrent_vsUniformThresh} additionally  shows the sum average cost $\Gamma(\pi^{us}_{np})$ for a uniform sampling policy $\pi^{us}_{np}$, which is identical to $\pi^{us}$ except that instead of using the optimal prediction strategy \eqref{eq:optpred}, it assumes the next state is identical to the last queried state. Clearly, $\pi^{us}_{np}$ performs worse among all policies, illustrating the importance of prediction strategy \eqref{eq:optpred}.

\begin{figure} 
	\centerline{\includegraphics[width=0.8\linewidth]{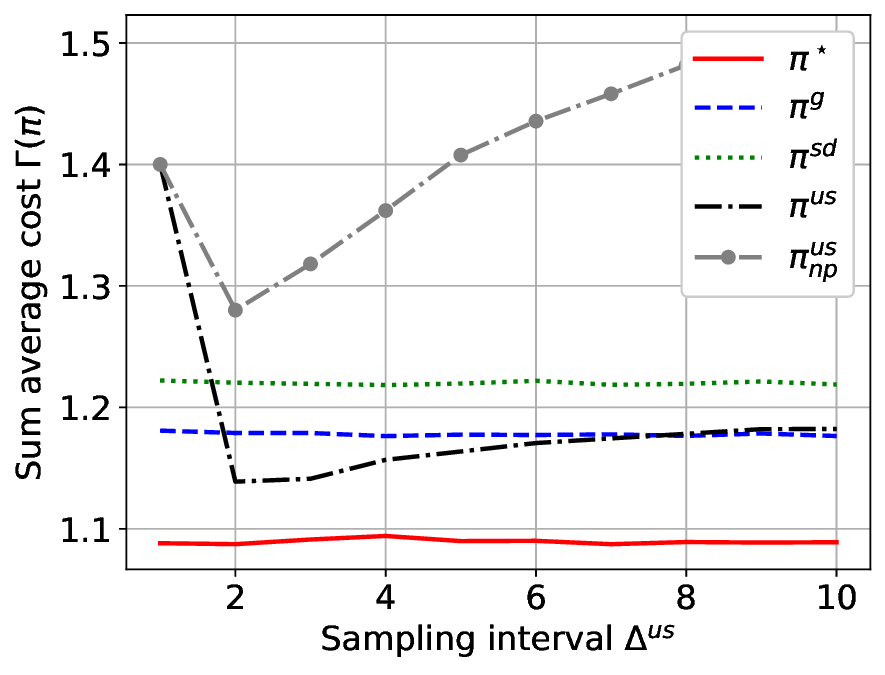}}
    \vspace{-1em}
	\caption{Performance of uniform sampling policy w.r.t. sampling interval, and prediction strategy.\vspace{-2ex}}
	\label{fig:recurrent_vsUniformThresh}
\end{figure}

Finally, to ensure that our observations are not tailored to few specific Markov chains, we simulate the policies for multiple randomly generated transition probability matrices $\cP$ (other parameters identical to the previous simulation). 
Interestingly, we find that in this case, as shown in Figure \ref{fig:vsP}, the performance  of the greedy policy $\pi^g$ is close to $\pi^\star$. 
This could be attributed to the fact that in each iteration, the elements of matrix $\cP$ are identically distributed.  
Therefore, the uncertainty in the state trajectory from each state is similar.  

Thus, comparing the result in Figure \ref{fig:vsP} with those of previous simulations, we find that $\pi^\star$ is most useful when $\cP$ is such that from any state the node is most likely to transition to a small subset of states than the others, i.e., the uncertainty regarding future state is low. Otherwise, $\pi^g$ is a better choice due to its simplicity and ease of implementation. 

\begin{figure} 
	\centerline{\includegraphics[width=0.8\linewidth]{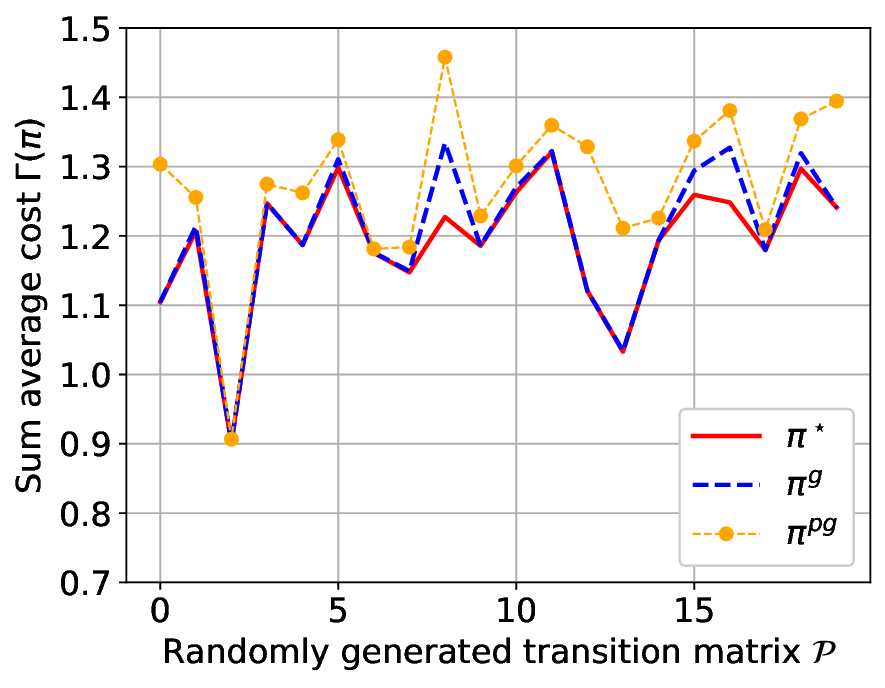}}
    \vspace{-1em}
	\caption{For randomly generated matrices $\cP$, the performance of greedy policy $\pi^g$ is close to the optimal policy $\pi^\star$.\vspace{-2ex}}
	\label{fig:vsP}
\end{figure}

\section{Conclusion} \label{sec:conclusion}
For the considered problem pertaining to the query-predict tradeoff, we found that in the Markovian setting, a state-based threshold policy $\pi^\star$ is optimal. However, determining the optimal thresholds is computationally expensive, especially, in dynamic settings where the thresholds need to be updated frequently (e.g. when certain parameters are unknown, and need to be learnt from observations). Therefore, we considered a greedy policy $\pi^g$, which is easy to implement, even in non-Markovian/dynamic settings. We saw that despite its simplicity, the greedy policy is near-optimal when the transition probabilities across states are identically distributed, e.g. in settings with limited prior knowledge of the environment or the state transitions. This is particularly relevant in modern systems where state-transitions is often non-Markovian and highly complex to be modeled accurately. However, in setting with absorbing states (or groups of states) $\pi^g$ must be implemented carefully, with periodic/random queries to improve prediction quality when the node is in one such state.

\bibliographystyle{IEEEtran} 
\bibliography{refs.bib}

\end{document}